\documentclass[11pt]{article}
\usepackage{amsmath,amssymb,amsfonts}
\usepackage{graphicx}
\usepackage{fullpage}
\usepackage{amsmath}
\usepackage{times}
\usepackage{float}

\floatstyle{ruled}
\newfloat{algorithm}{h}{loa}
\floatname{algorithm}{Algorithm}

\newtheorem{definition}{Definition}
\newtheorem{theorem}{Theorem}
\newtheorem{corollary}{Corollary}
\newtheorem{lemma}{Lemma}
\newtheorem{note}{Note}%
\newcommand{\qed}{\mbox{\ \ \ }\rule{6pt}{7pt}}%
\newenvironment{proof}{\par{\bf Proof:}}{\qed \par}

\renewcommand{\P}{\mathbb P} 
\renewcommand{\Pr}{\mathbb P} 
\newcommand{\R}{\mathit R} 
\newcommand{\reals}{\mathbb{R}} 

\newcommand{\E}{\mathbb E}

\newcommand{\cY}{{ Y}}

\newcommand{\err}{\mathrm{err}}

\newcommand{\anglesep}{\theta}

\newcommand{\sgn}{{\mathrm{ sign}}}

\newcommand{\gammat}{\tilde{\gamma}}
\newcommand{\tgamma}{\gammat}
\newcommand{\distf}{\mathrm{d}}
\renewcommand{\P}{\mathbb P} 

\newcommand{\ta}{\tilde{a}}
\newcommand{\tb}{\tilde{b}}
\newcommand{\tw}{\tilde{w}}
\newcommand{\tc}{\tilde{c}}

\newcommand{\tA}{\tilde{A}}
\newcommand{\tW}{\tilde{W}}
\newcommand{\tU}{\tilde{U}}
\newcommand{\tC}{\tilde{C}}

\newcommand{\tk}{\tilde{k}}
\newcommand{\tr}{\tilde{r}}
\newcommand{\ttk}{\hat{k}}
\newcommand{\tu}{\tilde{u}}

\newcommand{\ttr}{\hat{r}}

\newcommand{\sign}{{\mathrm{ sign}}}

\newcommand{\tO}[1]{\tilde{O}{#1}}

\newcommand{\talpha}{\alpha}

\def\Pr{{\sf P}}\def\E{{\sf E}}
\def\eps{\varepsilon}
\def\R{\Re}

\def\span{{\sf span}}

\newcommand{\vars}{\mathrm{vars}}

\newcommand{\TS}{\mathrm{TS}}
\newcommand{\TTS}{\tilde{\mathrm{TS}}}

\newcommand{\labell}{\mathrm{r}}

\newcommand{\thh}{\tilde{m}}

\newcommand{\tM}{\tilde{M}}
\newcommand{\tTS}{\mathrm{conj}(\tilde{M})}

\newcommand{\wepsacc}{\epsilon_{acc}}
\newcommand{\uepsacc}{\tilde{\epsilon}_{acc}}
\newcommand{\smdim}{\tau}

\title{Efficient Representations for Life-Long Learning and Autoencoding}

\author{{Maria-Florina Balcan}\\{CMU}\\{\tt ninamf@cs.cmu.edu} \and {Avrim Blum}\\{CMU}\\{\tt avrim@cs.cmu.edu} \and {Santosh Vempala}\\{Georgia Tech}\\{\tt vempala@cc.gatech.edu}}

\begin{document}

\addtocounter{page}{-1}
\maketitle
\thispagestyle{empty}
\begin{abstract}
It has been a long-standing goal in machine learning, as well as in AI more generally, to develop life-long learning systems that learn many different tasks over time, and reuse insights from tasks learned, ``learning to learn'' as they do so.  In this work we pose and provide efficient algorithms for several natural theoretical formulations of this goal.  Specifically, we consider the problem of learning many different target functions over time, that share certain commonalities that are initially unknown to the learning algorithm.  Our aim is to learn new internal representations as the algorithm learns new target functions,  that capture this commonality and allow subsequent learning tasks to be solved more efficiently and from less data. We develop efficient algorithms for two very different kinds of commonalities that target functions might share: one based on learning common low-dimensional and unions of low-dimensional subspaces and one based on learning  nonlinear Boolean combinations of features.  Our algorithms for learning Boolean feature combinations additionally have a dual interpretation, and can be viewed as giving an efficient procedure for constructing near-optimal sparse Boolean autoencoders under a natural ``anchor-set'' assumption.
\end{abstract}
\newpage

\section{Introduction}

Machine learning has developed a deep mathematical understanding as well as powerful practical methods for the problem of learning a single target function from large amounts of labeled data. Yet if we wish to produce machine learning systems that persist in the world, we need methods for continually learning many tasks over time and that,  like humans \cite{Gop01}, improve their ability to learn as they do so, needing less data (per task) as they learn more.  A natural approach for tackling this goal (called ``life-long learning''~\cite{Thrun96b,ThrunM95}  or ``transfer learning''~\cite{aep08,pm:13} or ``learning to learn''~\cite{Baxter97,TP97}) is to use information from previously-learned tasks to {\em  improve the underlying representation} used by the learning algorithm, under the hope or belief that some kinds of commonalities across tasks exist.  These commonalities could be a single low-dimensional or sparse representation, a collection of multiple low-dimensional or sparse representations, or some combination or hierarchy, such as in Deep Learning~\cite{bengio13,bengio11}.  In this paper, we develop algorithms with provable efficiency and sample size guarantees for several interesting categories of commonalities, considering both linear and Boolean transfer functions, under natural distributions on the data points.

Specifically, we consider a setting where we are trying to solve a large number of binary classification problems that arrive one at a time. Each classification problem will individually be learnable from a polynomial-size sample,\footnote{Except in Section \ref{se:poly} where we consider learning multivariate polynomials and allow membership queries.} but our goal will be to learn new internal representations that will allow us to learn new target functions faster and from less data.   We will furthermore aim to do this in a streaming setting in which we cannot keep the labeled data for problem $t$ in memory when we move on to problem $t+1$, only the learned hypotheses (which will be required to have a compact description) and the current internal representation. 

We start by considering a conceptually simple case that each classification problem is a linear separator, and that what their associated target vectors share in common is  they  lie
 in a low $k$ dimensional subspace of the ambient space $\R^n$ (equivalently, there exist a small number of hidden linear metafeatures and each target is a linear separator over these metafeatures).
This case has been considered in the ``batch'' setting in which one has data available for all target functions at the same time and therefore can solve a joint optimization problem~\cite{aep08,pm:13}.  However,
for the online setting, a key challenge is that we won't have perfectly learned the previous target functions when we set out to learn our next one.\footnote{In particular, because of this we will need to be particularly careful with which targets we use in constructing our metafeatures, as well as in controlling the propagation of errors.  See, e.g.,  Lemma \ref{lem:subspace} and Figure \ref{figure}.}  For this problem we provide a sample-efficient polynomial time algorithm that, when the underlying data distributions for our $m$ learning problems are log-concave, has labeled sample complexity much better than the $\Omega(nm/\epsilon)$ sample complexity  required for learning the tasks separately.

We then consider scenarios where the commonalities require a representation with more than one level of metafeatures and provide efficient algorithms for these settings as well.
For linear metafeatures, we provide a natural framework where two levels of metafeatures can be efficiently extracted and provide substantial benefit:  specifically, we analyze a scenario where the target functions all
 lie in a $k$ dimensional space and  furthermore within that  $k$-dimensional space,  each target
 lies in one of $r$ different constant dimensional spaces, where $r$ could be large. This models situations where there are really $r$ different {\em types}
 of learning problems but  they do share some commonalities across types (given by a low $k$-dimensional subspace). 

In Section~\ref{se:monomial} we develop algorithms for a scenario where the metafeatures are non-linear, in particular where features are boolean and the metafeatures are products. 
We give an efficient algorithm for finding the fewest product-based metafeatures for a given set of target monomials under an ``anchor-variable'' assumption analogous to the anchor-word assumption of \cite{AroraGM12}, and prove bounds on its performance for learning a series of target functions arriving online.  We then give an extension that learns an approximately-optimal overcomplete  sparse representation (we may have more metafeatures than input features, but each target should have a sparse representation) under a weaker form of assumption we call the ``anchor-set'' assumption (anchor variables no longer make sense in the overcomplete case).  These results can be viewed as giving efficient algorithms for a Boolean autoencoding where given a set of black-and-white pixel images (vectors in $\{0,1\}^n$) we want to find either (a) the fewest ``basic objects'' (also vectors in $\{0,1\}^n$) such that each given image can be reconstructed by superimposing some subset of them (taking their bitwise-OR), or (b) a larger number of such objects such that each image can be reconstructed by superimposing only a few of them.  In the first case our algorithm finds the optimal solution under the anchor-variable assumption (the problem is NP-hard in general) and in the second case it finds a bicriteria approximation (for a given sparsity level, approximates both the number and the sparsity to logarithmic factors) under the weaker anchor-set assumption.

In Section~\ref{se:poly} 
we show how our results can be applied to the case that target functions are polynomials of low $L_1$ norm whose terms share pieces in common (within and across polynomials), a scenario that can be expressed via two levels of product-based metafeatures. Interestingly, as opposed to the algorithms for linear metafeatures, this algorithm periodically re-compactifies its current representation.  In particular, whenever a new polynomial cannot be learned using the current representation and must be learned from scratch, we then revisit the previously-learned polynomials and optimally ``compactify'' them into the fewest number of (possibly
overlapping) conjunctive metafeatures that can be used to recreate all their monomials.

\subsection{Related Work}
Most related work in multi-task or transfer learning considers the case that  all target functions are present simultaneously or that target functions are drawn from some easily learnable distribution.
Baxter~\cite{Baxter97,baxter00} developed some of the earliest  foundations for transfer learning, by providing sample complexity results for achieving low average error in such settings. Other related sample complexity results appear in~\cite{sbd:03}.

Recent work of ~\cite{pm:13,halnips12} considers the problem of learning multiple linear separators that share a common low-dimensional subspace in 
the batch setting where all tasks are given up front. They specifically provide guarantees for a natural ERM algorithm with trace norm regularization.  There has also been work on applying the Group Lasso method to batch multi-task learning which solves a specific multi-task optimization problem \cite{RV12}.
By contrast with these results, our setting is more demanding since we aim to achieve small error on all tasks and to do so online without keeping all training data from past learning tasks in memory. 

\cite{ccg10} considers multi-task learning  where explicit known relationships among tasks are exploited for faster learning.  In their setting each learning problem is an online problem but the collection of learning problems are all occurring simultaneously.  Discussion in~\cite{valiant2000} hints toward the type of the algorithms we analyze in Section~\ref{se:linear}, but without formal analysis about how the error accumulation could harm the sample complexity (which, as we will see, is one of the central challenges in this setting).

The problem of trying to learn invariants or other commonalities when faced with a series of learning tasks arriving over time has a long history in applied machine learning (e.g., \cite{Thrun96b,ThrunM95}).
Our work is the first to give provable efficiency guarantees for learning  multi-layer representations in this  life-long learning setting.

\section{Preliminaries}
We assume we have $m$ learning (binary classification) problems that arrive online over time.  The learning problems are all over a common instance space $X$ of dimension $n$ (e.g., we will consider $X=\R^n$ and $X=\{0,1\}^n$) but each has its own target function and potentially its own distribution $D_i$ over $X$.  Formally, learning problem $i$ is defined by a distribution $P_i$ over $X \times \cY$ where $\cY = \{-1,1\}$ is the label space and $D_i$ is the marginal over $X$ of $P_i$, and the goal of the learning algorithm on problem $i$ is to produce a hypothesis function $h_i$  of small error, where
$\err(h_i) = \err_{P_i}(h_i) = P_{(x,y) \sim P_i}  [h_i(x) \neq y].$

\section{Life-long Learning of Halfspaces}
\label{se:linear}
We consider here the natural case that $X=\R^n$ and each target function is a linear separator
going through the origin; that is, for all $i$ there exists  $a_i$ of unit length such that
for all $(x,y)$ drawn from $P_i$ we have $\sgn(a_i\cdot x) = y$. 
To begin, we will assume that what the target functions have in common is that they all lie in some common $k$-dimensional subspace of $\R^n$ for $k \ll \min(n,m)$.  In particular, let $A$ be a $m$ by $n$ matrix whose rows are $a_1, \ldots, a_m$; then our assumption is that $A$ has rank $k$.
This implies that there exist a decomposition  $A=C W$, where  $W$ is a $k \times n$ matrix and $C$ is a $m \times k$ matrix.  The rows $w_1, \ldots, w_k$ of $W$ can be viewed as $k$  {\em linear metafeatures} that are sufficient to describe all $m$ learning problems, or equivalently we can view this as a network with one middle layer of $k$ hidden linear units.
In fact, our algorithms will work under a more robust condition that allows for the $a_i$ to be ``near'' to a low-dimensional subspace (see Theorem \ref{thm:linmetalogconc}).

In this section we analyze the following online algorithm for this setting; note, the algorithm is very natural but the challenge will be to analyze it and control the propagation of error at reasonable sample sizes. Let $\eps_{acc}$ be a quantity to be determined later.  For the first learning problem  we just learn it to error $\eps_{acc}$ using the original input features and let the resulting weight vector be $\tw_1$. Suppose now we have produced weight vectors $\tw_1, \ldots, \tw_{k'}$ and we are considering problem $i$.  We will first see if we can learn problem $i$ well (to error $\epsilon$) as a linear combination of the $\tw_j$.  If so, then we mark this as a success and go on to problem $i+1$.  If not, then we will learn it to error $\eps_{acc}$ using the input features and add the hypothesis weight vector as $\tw_{k'+1}$.
(See Algorithm~\ref{algo-one} for formal details). The challenge is how small $\eps_{acc}$ needs to be for this to succeed.

We show in the following that if the $D_i$ are isotropic log-concave (which includes many distributions such as Gaussian and uniform, see, e.g.,  \cite{LV07}),
the above procedure will be successful and learn the target functions with much fewer labeled examples in total than by learning each function separately.
We start with some useful facts and present a lemma (Lemma~\ref{lem:subspace}) that is crucial for our analysis.

Given two vectors $a$ and $b$ and a distribution $\tilde{D}$, let
$ \distf_{\tilde{D}}(a,b)=\P_{x \sim \tilde{D}}(\sign(u \cdot x) \neq  \sign(v \cdot x))$.
 Let $\theta(a,b)$ be the angle between two vectors $a$ and $b$. For a vector $a$ and a subspace $V$, let $\theta(a,V) = \min_{b \in V} \theta(a,b)$ be the angle between $a$ and its closest vector in $V$  (in angle).  For subspaces $U$ and $V$, let
$\theta(U,V) = \max_{u \in U} \theta(u,V)$. That is, $\theta(U,V) \leq \alpha$ iff for all $u \in U$ there exists $v \in V$ such that $\theta(u,v) \leq \alpha$.

\begin{lemma}
\label{l:angle}
Assume $D$ is an isotropic log-concave in $R^n$. Then there exist constants $c$ and $c'$ such that for any two unit vectors $u$ and $v$ in $\reals^d$  we have
$ c \anglesep(v,u) \leq \distf_D(u,v) \leq c' \anglesep(v,u).$
\end{lemma}

\begin{proof} The proof of the lower bound appears in~\cite{BalcanLong:13}. The proof of the upper bound is implicit in the earlier work of~\cite{Vem10} -- we provide it here for completeness.
The key idea is to project the region of disagreement  in the space given by the two normal vectors, and then using properties of log-concave distributions in 2-dimensions.
Specifically, consider the plane determined by $u$ and $v$, and let $g$ be the
2-dimensional marginal of the density function over this plane. Then $g$ is an isotropic and log-concave density function over $R^2$.

It is known~\cite{KLT09} that for some constants $k_3,k_4$ we have
$g(z) \leq k_3 e^{- k_4 ||z||}.$
Given this fact, we just need to show that the integral of $k_3e^{- k_4
||z||}$ over the region $\{z:u \cdot z \geq 0, v \cdot z \leq 0\}$
is at most $c'\alpha$ for some constant $c'$, where $\alpha$ is the
angle between $u$ and $v$ (the integral over $\{z:u \cdot z \leq 0, v
\cdot z \geq 0\}$ is analogous).  In particular, using polar
coordinates, we can write the integral as:
$$\int_{\theta=0}^{\alpha} \int_{r=0}^{\infty} f(r\cos \theta, r \sin
\theta)\; r \; dr d\theta \leq \int_{\theta=0}^{\alpha} \int_{r=0}^{\infty} r k_3 e^{-k_4 r}dr
d\theta.$$
The inner integral evaluates to a constant
 and
therefore the entire integral is bounded by $c'\alpha$ for some
constant $c'$ as desired.
\end{proof}

Lemma \ref{l:angle} implies that if we learn some target $a_i$ to error $\epsilon_{acc}$, then the angle between our learned vector and the target will be $O(\epsilon_{acc})$.  In the other direction, if the target lies in subspace $W$ and we have learned a subspace $\tilde{W}$ such that $\theta(W,\tilde{W})$ is small, then there will exist a low-error weight vector in $\tilde{W}$.   Ideally, we would therefore like to say that if we construct a subspace $\tilde{W}$ out of vectors $\tilde{w}_i$ that individually are close to their associated targets $w_i$, then $\tilde{W}$ is close to the span $W$ of the $w_i$.  Unfortunately, this is not in general true if the targets are close to each other, e.g., see Figure \ref{figure}(a).  We will address this by using the fact that each $\tilde{w}_i$ was {\em not} learnable using the span of the previous $\tilde{w}_j$.  We begin with a helper lemma (Lemma \ref{lem:one-vec}), which can be viewed as addressing the special case that all previous targets have been learned {\em perfectly}, and then present our main lemma (Lemma \ref{lem:subspace}).

\begin{figure}
\centerline{\includegraphics{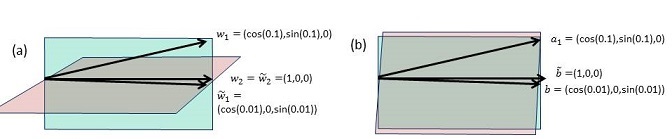}}
\caption{{(a) Even though each $\tilde{w}_i$ is within angle 0.11 of its corresponding $w_i$ ($\theta(w_2,\tilde{w}_2)=0$, $\theta(w_1,\tilde{w}_1) \leq \theta(w_1,w_2) + \theta(w_2,\tilde{w}_1) = 0.11$), the two subspaces are orthogonal ($\span(w_1,w_2)$ is the $x$-$y$ plane and $\span(\tilde{w}_1,\tilde{w}_2)$ is the $x$-$z$ plane).  
(b) For intuition for Lemma \ref{lem:one-vec}: now, the two subspaces $V,\tilde{V}$ {\em are} close in angle, with angle at most $\frac{\pi}{2}$ times the ratio of $\theta(\tilde{b},b) = 0.01$ to $\theta(\tilde{b},a_1) = 0.1$. \label{figure}}}
\end{figure}

\begin{lemma}\label{lem:one-vec}
Let $U=\span\{a_1, \ldots, a_{k-1}\}, V = \span\{a_1, a_2, \ldots, a_{k-1}, b\}$ and $\tilde{V} = \span\{a_1,\ldots, a_{k-1}, \tilde{b}\}$ be sets of vectors in $\R^n$. Then, 
\[
\theta(V,\tilde{V}) \le \frac{\pi}{2}\frac{\theta(\tilde{b},b)}{\theta(\tilde{b},U)}.
\]
\end{lemma}
\begin{proof}
(For intuition, see Figure \ref{figure}(b).)
First, we may assume $b,\tilde{b} \not\in U$ because if $b \in U$ then $\theta(V,\tilde{V})=0$ and if $\tilde{b} \in U$ then $\theta(\tilde{b},U)=0$.  Additionally we may assume $b$ and $\tilde{b}$ are unit-length vectors since we are interested only in angles.  Next, let $u \in V$ be the unit-length vector in $V$ of farthest angle from $\tilde{V}$, i.e., $\theta(u,\tilde{V}) = \theta(V,\tilde{V})$.  We can write $u$ as a linear combination of $b$ and some vector $u_1 \in U$, and will prove the lemma by showing there must be some nearby vector in the span of $u_1$ and $\tilde{b}$.  Specifically, using the fact that $\theta(V,\tilde{V}) = \theta(u,\tilde{V}) \leq \theta(\span(u_1,b),\span(u_1,\tilde{b}))$ and the fact that $\theta(\tilde{b},U) \leq \theta(\tilde{b},u_1)$, it is sufficient to prove that:
$$\theta(\span(u_1,b),\span(u_1,\tilde{b})) \leq \frac{\pi}{2}\frac{\theta(b,\tilde{b})}{\theta(\tilde{b},u_1)},$$
which is just a statement about 3-d space.  Let $\alpha = \theta(\span(u_1,b),\span(u_1,\tilde{b}))$ and $\beta = \theta(\tilde{b},u_1)$.  We can wlog write $u_1 = (1,0,0)$ and assume $\span(u_1,b)$ is the $x$-$y$ plane.   Since $\span(u_1,\tilde{b})$ has angle $\alpha$ with the $x$-$y$ plane and intersects the $x$-axis, we can write  $\tilde{b} = \cos(\beta)u_1 + \sin(\beta)u_2$ for $u_2 = (0, \cos(\alpha), \sin(\alpha))$. Now, $\theta(b,\tilde{b})$ is at least $\sin^{-1}$ of the Euclidean distance of $\tilde{b}$ to the $x$-$y$ plane, which is $\sin(\alpha)\sin(\beta)$.   The lemma then follows from the fact that $\alpha\beta \leq \frac{\pi}{2}\sin^{-1}(\sin(\alpha)\sin(\beta))$ for $0 \leq \alpha,\beta \leq \frac{\pi}{2}$.
\end{proof}

The key point of the next  lemma is that the errors (angles between $a_i$ and $\ta_i$) only contribute additively to the overall angle gap between subspaces so long as each new learned vector is sufficiently far from the previously-learned subspace.   In contrast, 
a difficulty with the usual analysis of perturbation of matrices is that while we can assume that each new $\tilde{a}_i$ is far from the span of the previous $\tilde{a}_1,\ldots, \tilde{a}_{i-1}$, we do not have control over its distance to the span of the past and future vectors $\{\tilde{a}_1, \ldots, \tilde{a}_{i-1}, \tilde{a}_{i+1}, \ldots, \tilde{a}_k\}$ as in the definition of the {\em height} of a matrix (e.g., \cite{Spielman}).
Note also that even adding the {\em same} vector to two different subspaces can potentially {increase} their angle (e.g., in Figure \ref{figure}(a), $\theta(w_1,\tilde{w}_1) < 0.11$ but $\theta(\span(w_1,w_2),\span(\tilde{w}_1,w_2) = \pi/2$). 

\begin{lemma}
\label{lem:subspace}
Let $V_k = \span\{a_1, \ldots, a_k\}$ and $\tilde{V}_k = \span\{\tilde{a}_1, \ldots, \tilde{a}_k\}$. Let $\epsilon_{acc}, \gamma \ge 0$ and $\epsilon_{acc} \le \gamma^2/(10k)$. Assume
for $i=2, \ldots, k$ that $\theta(\tilde{a}_i,\tilde{V}_{i-1}) \ge \gamma$, and for $i=1,\ldots, n$, $\theta(a_i, \tilde{a}_i) \le \epsilon_{acc}$.
Then
\[
\theta(V_k,\tilde{V}_k) \le 2k\frac{\epsilon_{acc}}{\gamma}.
\]
\end{lemma}

\begin{proof}
The proof is by induction on $k$, on the stronger hypothesis that the conclusion holds for $V_k = \span\{W,a_1, \ldots,a_k\}$ and
$\tilde{V}_k = \span\{W,\tilde{a}_1,\ldots, \tilde{a}_k\}$ for any fixed subspace $W$.  Note that the base case ($k=1$), follows directly from Lemma \ref{lem:one-vec}, using $W=\tilde{V}_{k-1}=U$, $\tilde{a}_1 = \tilde{b}$, and $a_1 = b$.  Now, let $V_k' = \span(V_{k-1},\tilde{a}_k)$.  Then we have:
\begin{eqnarray*}
\theta(V_k,\tilde{V}_k) &\le& \theta(V_k,V_k') + \theta(V_k',\tilde{V}_k)\\
& & \mbox{[by triangle inequality]}\\
&\le&\frac{\pi}{2}\frac{\theta(\tilde{a}_k,a_k)}{\theta(\tilde{a}_k,V_{k-1})} + \frac{2(k-1)\epsilon_{acc}}{\gamma}\\
& &\mbox{[the first term is by Lemma \ref{lem:one-vec}, and the second term is by induction using $W = \span(\tilde{a}_k)$]}\\
&\le&\frac{\pi}{2}\frac{\epsilon_{acc}}{\theta(\tilde{a}_k,\tilde{V}_{k-1})-\theta(V_{k-1},\tilde{V}_{k-1})} + \frac{2(k-1)\epsilon_{acc}}{\gamma}\\
& & \mbox{[by triangle inequality: $\theta(\tilde{a}_k,\tilde{V}_{k-1}) \leq \theta(\tilde{a}_k,V_{k-1}) + \theta(V_{k-1},\tilde{V}_{k-1})$]}\\
&\le&\frac{\pi}{2}\frac{\epsilon_{acc}}{\gamma - \frac{2(k-1)\epsilon_{acc}}{\gamma}} + \frac{2(k-1)\epsilon_{acc}}{\gamma}\\
& & \mbox{[by assumption and by induction]}\\
&\le&\frac{\epsilon_{acc}}{\gamma}\left(\frac{\pi}{2}\frac{\gamma^2}{\gamma^2 - 2(k-1)\epsilon_{acc}} + 2(k-1)\right)\\
&\le& 2k\epsilon_{acc}/\gamma, 
\end{eqnarray*}
where the last step comes from using $\epsilon_{acc} \le \gamma^2/(10(k-1))$.
\end{proof}

\begin{algorithm}
{\bf Input}: $n$,$m$,$k$, access to labeled examples for problems $i \in \{1, \ldots, m\}$, parameters $\epsilon$ and $\epsilon_{acc}$.
\begin{enumerate}\setlength{\itemsep}{0pt}
\item Learn the first target to error $\epsilon_{acc}$ to get an $n$-dimensional vector $\talpha_1$.  Set $\tw_1 = \talpha_1$; $\tk=1$ and $i_1=1$.
\item For the learning problem  $i=2 $ to $m$
\begin{itemize}\vspace*{-0.1in}
\item Attempt to learn using the representation $v \rightarrow (\tw_1 \cdot v, ...,\tw_{\tk} \cdot v)$.
I.e., check if for learning problem $i$ there exists a hypothesis
$ \sgn(\alpha_{i,1} (\tw_1 \cdot v)+ \cdots + \alpha_{i,\tk} (\tw_{\tk} \cdot v))$ of error at most $\epsilon$.
\begin{itemize}
\item[(a)] If yes, set $\tc_i=(\alpha_{i,1}, \ldots, \alpha_{i,\tk} , 0, \ldots, 0)$.
\item[(b)] If not,  learn  a classifier $\talpha_i$ for problem $i$ of accuracy $\epsilon_{acc}$ by using the original features.
Set $\tk=\tk+1 $,  $i_{\tk}=i$ ,
  $\tw_{\tk} = \talpha_i$, and  $\tc_i = e_{\tk}$.
  \end{itemize}
\end{itemize}

\item Let $\tW$ be an $\tk \times n$ matrix whose rows are $\tw_1, \ldots, \tw_{\tk}$ and let $\tC$ be the matrix $m \times \tk$ matrix whose rows are
 $\tc_1, \ldots, \tc_{\tk}$. Compute $\tA=\tC \tW$.
\end{enumerate}
{\bf Output}: $m$ predictors; predictor $i$ is $v \rightarrow \sgn(\tA_i \cdot v)$
\caption{Life-long learning of halfspaces sharing a common low-dimensional subspace\label{algo-one}}
\end{algorithm}

We now put these together to analyze Algorithm \ref{algo-one} when target functions lie on, or close to, a low-dimensional subspace. Specifically, say that a subsequence of target functions $a_{i_1}, a_{i_2}, \ldots$ is $\gamma$-separated if each $a_{i_j}$  has angle greater than $\gamma$ from the span of the previous $a_{i_1}, \ldots, a_{i_{j-1}}$.  Define the {\em $\gamma$-effective dimension} of targets $a_1, a_2, \ldots, a_m$ as the size of the largest $\gamma$-separated subsequence.   Our assumption will be that the 
$\gamma$-effective dimension of the targets is at most $k$ for $\gamma = c\epsilon$ for some absolute constant $c>0$, where $\epsilon$ is our desired error rate per target.  
Note that for $\gamma=0$,  $\gamma$-effective dimension equals the dimension of the subspace spanned, and for $\gamma>0$ this allows the targets to just be ``near'' to a low-dimensional subspace.

\begin{theorem}  Assume that all marginals $D_i$ are isotropic log-concave. 
Choose $\gamma = c_1\epsilon$ and $\epsilon_{acc}$ s.t.
$2k\frac{\eps_{acc}}{\gamma} + \gamma = c_2\epsilon$ 
for sufficiently small constants $c_1,c_2>0$.
Consider running Algorithm~\ref{algo-one} 
with parameters $\epsilon$ and $\epsilon_{acc}$ on any sequence of targets whose $\gamma$-effective dimension is at most $k$.
Then $\tk \leq k$ (the rank of $\tA$ is at most $k$).
Moreover the total   number of labeled examples needed to learn all the problems to error $\epsilon$
is $\tO(nk/\epsilon_{acc}+ k m /\epsilon) = \tO(nk^2/\epsilon^2 + km/\epsilon)$.
\label{thm:linmetalogconc}
\end{theorem}

\begin{proof}
We divide problems in two types: problems of  type (a) are those for which we can learn a classifier of error at most
 $\epsilon$ by using the previously learnt problems; the rest are of type (b).

For problems of type (a) we achieve  error $\epsilon$ by design.
For each problem $i$ of type (b) we open a new row in $\tW$, and set $\tw_{\ttk}=\alpha_i$, where $\ttk$ is such that $i_{\ttk}=i$.  We also set $\tc_i = e_{\ttk}$, so  $\ta_{i}= \alpha_i$. Since $\alpha_i$ has error at most $\epsilon_{acc}$, we have $\theta(\tw_{\ttk},a_{i_{\ttk}}) \leq \epsilon_{acc}/c$ for some absolute constant $c$ (by Lemma \ref{l:angle}).

We next  show that $\tk \leq k$.
We prove by induction that for each $\tw_{\ttk}$ we create for a problem $i=i_{\ttk}$, we have both (1) $a_{i_{\ttk}}$ is $\gamma$-far from
$\span\{a_{i_1}, \cdots ,a_{i_{\ttk-1}}\}$   and  (2)  $\tw_{\ttk}$ is $\gamma$-far from $\span(\tw_1,...,\tw_{\ttk-1})$.

Step $\ttk=1$ follows immediately.
For the inductive step  $\ttk>1$: if we create  $\tw_{\ttk}$ for a problem $i=i_{\ttk}$,
 this only
happens if there is no vector  in the span of the previous metafeatures
$\tw_j$, $j<i$ that has error less than $\epsilon$ for problem $i_{\ttk}$.\footnote{Technically, since we are learning over a finite sample, we can only be confident that there is no vector in the span of error at most $\epsilon/2$.  However, we can absorb these factors of 2 into the constants $c,c'$.}
That is $a_{i_{\ttk}}$ is at least $\epsilon/c'$-far from the   $\span \{\tw_1,...,\tw_{\ttk-1}\}$ for some absolute constant $c'$ (by Lemma \ref{l:angle}). We also have $\theta(\tw_{\ttk},a_{i_{\ttk}}) \leq \eps_{acc}/c$, therefore,  by triangle-inequality, we obtain 
$$\theta(\tw_{\ttk}, \span(\tw_1,...\tw_{\ttk-1})) \geq \epsilon/c' - \eps_{acc}/c \geq \gamma.$$
Thus $\tw_{\ttk}$ is $\gamma$-far from $\span\{\tw_1,...,\tw_{\ttk-1}\}$.
It remains to show that $a_{i_{\ttk}}$ is $\gamma$-far from the
span of $\{a_{i_1}, \cdots ,a_{i_{\ttk-1}}\}$.
Suppose for contradiction 
that $\anglesep(a_{i_{\ttk}},
 \{a_{i_1}, \cdots ,a_{i_{\ttk-1}}\}) \leq \gamma$.
We will show that this implies there exists $\tb_{i_{\ttk}} \in \span\{\tw_{1}, \cdots , \tw_{\ttk-1} \}$ with  error at most $\epsilon$, contradicting the fact that
no such vector exists.

 By construction we have $\theta(a_{i_j}, \tw_j) \leq \epsilon_{acc}/c$ for $j\in \{1, \ldots, \ttk-1\}$; also by induction we have
 $\tw_{{j}}$ is $\gamma$-far from the
span of $\{\tw_{1}, \cdots ,\tw_{{j-1}}\}$ for $j \in \{1, \ldots, \ttk-1\}$. By Lemma~\ref{lem:subspace} we obtain that
$$\theta(\span\{a_{i_1}, \cdots ,a_{i_{\ttk-1}}\},\span\{\tw_{1}, \cdots ,\tw_{\ttk-1}\}) \leq 2k{\eps_{acc}}/{(c\gamma)}.$$
These together with triangle inequality imply that
$$\theta( a_{i_{\ttk}},\span\{\tw_{1}, \cdots ,\tw_{{\ttk-1}}\}) \leq \gamma + 2k\frac{\eps_{acc}}{c\gamma}  \leq \epsilon/c'.$$
So by Lemma \ref{l:angle} there exist  $\tb_{i_{\ttk}} \in \span\{\tw_{1}, \cdots ,\tw_{{\ttk-1}}\}$ of  error at most $\epsilon$, which contradicts our assumption.  Therefore, our induction is maintained (by condition (2)) and so we have $\tk \leq k$ (by condition (1) and our assumption on the $\gamma$-effective dimension of the  targets).
\end{proof}

By setting $\gamma=O(\epsilon)$ and $\epsilon_{acc} = O(\epsilon^2/k)$ the total   number of labeled examples needed to learn all the problems to error $\epsilon$ is $\tO(nk^2/\epsilon^2+ k m /\epsilon)$, which could be significantly lower than learning each problem separately.
In this case the sample complexity would be $\Omega(mn/\epsilon)$ even under log-concave distributions~\cite{BalcanLong:13}.

\begin{note}
As stated, Algorithm \ref{algo-one} is not efficient because it requires finding an optimal linear separator in Step 2, which in general is hard.  However, for log-concave distributions, there exist algorithms running in time poly$(k,1/\epsilon)$ that find a near-optimal linear separator: in particular, one of error $\epsilon$ under the assumption that the optimal separator has error $\eta = \epsilon/\log^2(1/\epsilon)$ \cite{ABL14}, and with near-optimal sample complexity \cite{H13b,Y13}.  Thus, by reducing $\eps_{acc}$ by an $O(\log^2(1/\epsilon))$ factor, one can achieve the bounds of Theorem 1 efficiently.
\end{note}

\subsection{Halfspaces with more complex common structure}
\label{multiple-levels-linear}
In this section we consider life-long learning of halfspaces with more complex common structure, corresponding to a multi-layer network of linear metafeatures.  
It is at first not obvious how multiple levels of linear nodes could help: if the target vectors span a $k$-dimensional subspace, then to represent them with a multi-layer linear network, each layer would need to have at least $k$ nodes.
However, the numbers of nodes in the network do not tell the whole story: sample complexity of learning can also be reduced via sparsity.

Specifically, we assume now that the target functions all
 lie in a $k$ dimensional space and that furthermore within that  $k$-dimensional space,  each target
 lies in one of $r$ different $\smdim$-dimensional spaces. This naturally models settings where there are really $r$ different types
 of learning problems but  they share some commonality across type (given by the common $k$-dimensional subspace).\footnote{For instance, imagine a job-placement company whose goal is to
decide which people would do well in which job.  In this setting, we can measure a large
number of features for each person (e.g., based on how well they do on various tests).  There are then $k$ ``intrinsic qualities'' that are linear
combinations of these features.  E.g., ``quantitative reasoning'' might
be one linear combination, ``people skills'' and ``time management'' might
be others, etc., and really what is important about each person is
where they sit in this $k$-dimensional subspace.  Then,
different jobs might belong in different low-dimensional spaces within
this k-dimensional space, based on what is important for that job.
I.e., there are $r$ ``kinds'' of jobs, each of which has a
$\tau$-dimensional subspace that is relevant for it.
}
We can view this as a network with two hidden layers: the first layer given by vectors $w_1,w_2 ,\ldots, w_k$, and the second  layer given by $r$  $\smdim$-tuples
of vectors, $u_1^1, \ldots, u_1^{\smdim}$, ..., $u_r^{1}, \ldots, u_r^{\smdim}$, where  $u_i^1, \ldots, u_i^{\smdim}$  span one of $\smdim$-dimensional spaces. 
In other words, the first hidden layer captures the overall low dimensionality and the second hidden layer captures sparsity. We assume $r \ll m$ and  $k \ll n$ and that $\smdim$ is a constant.

Algorithmically, given a new problem we first try to learn  well via a
sparse linear combination of only $\smdim$  second level metafeatures 
If we fail, we try to learn based
on the  first  level metafeatures  and if successful we add a new second level metafeature corresponding to this target.
If that fails, we  learn using the input features and then we add both a first and second level metafeature corresponding to this target.
For log-concave distributions, by using the subspace lemma and an error analysis similar to that for Theorem~\ref{thm:linmetalogconc} we can show we have $\tk \leq k$ and $\tr \leq \smdim r$.
Formally:


\begin{theorem}  Assume  all marginals $D_i$ are isotropic log-concave and the target functions satisfy the above conditions. Consider $\tgamma \leq c\epsilon$, 
 $\uepsacc \leq c\frac{\tgamma \epsilon}{\smdim}$, $\gamma \leq c\uepsacc$,
and $\wepsacc \leq c\frac{\gamma\uepsacc}{k}$
for (sufficiently small) constant $c>0$.
Consider running 
Algorithm~\ref{algo-two} (see appendix)
with parameters $\epsilon$, $\wepsacc$, and $\uepsacc$.
Then $\tk \leq k$ and $\tr \leq \smdim r$.
Moreover the total   number of examples needed to learn all the problems to error $\epsilon$
is $\tO(nk/{\wepsacc} + kr/{\uepsacc} + m \log(r)/\epsilon)$.
\label{thm:linmetalogconc2levels}
\end{theorem}

(Proof in appendix).
By setting $\tgamma=\epsilon/2$, $\uepsacc=\Theta(\epsilon^2/\tau)$, $\gamma=\Theta(\epsilon^2/\tau)$, $\wepsacc=\Theta(\epsilon^4/\tau^2 k)$ we get that the total   number of labeled examples needed to learn all the problems to error $\epsilon$ is
 $\tO(nk^2 \tau^2/{\epsilon^4} + kr \tau^2/{\epsilon^2} + m \tau \log(r)/\epsilon)$.
This could be significantly lower than learning each problem separately or by learning the problems together but only using one layer of metafeatures.
Specifically, if we used one layer of  metafeatures as in  Theorem~\ref{thm:linmetalogconc} (corresponding to the $k$-dimensional subspace) the sample complexity would be  $O(nk^2/\epsilon^2 + mk/\epsilon)$. Alternatively we could have just one middle layer of size $r \tau$ and learn sparsely within that, but this would also give worse bounds if $r$ is large.
As a concrete example,  if $\epsilon$ is constant, $k = \sqrt{n}$, $r = n^2$ and $m = n^{2.5}$, we get that
the two-layer algorithm requires only 
$O(\tau^2/\epsilon^2+\tau \log(r)/\epsilon)$ examples per target. On the other hand, the other two options require
   at least $O(k/\epsilon)$ examples per target, which could be much worse.

\section{Life-long Learning of Monomials}
\label{se:monomial}
We now consider a nonlinear case where the metafeatures will be products and combined via products.  Specifically, we assume that the instance space is $X=\{0,1\}^n$, that the $m$ target functions are conjunctions (i.e., products) of features, and that there exist $k$  monomial metafeatures such that all the target functions can be expressed as conjunctions (products) over them.  Our goal will be to learn them efficiently.

If the  metafeatures do not overlap, then this can be viewed as an instance of the linear case.
Each target function   can be described by an indicator vector with coefficients in $\{0,1\}$ (plus a threshold that can be converted to an integer weight for a dummy variable $x_0$). More importantly, if the metafeatures do not overlap, then the indicator vectors for all the {\em targets} are in a space of rank $k$ with basis given by the indicator vectors of the metafeatures.
If furthermore the underlying distribution is one for which, when learning from scratch, we can learn the target functions exactly (e.g., a product distribution where each variable is set to 0 some non-negligible fraction of the time) then we can directly apply the analysis for linear case.  In fact, the overall analysis is much simpler since we have the targets exactly that were learned from scratch. 

So, the interesting case is when metafeatures may overlap (it is easy to construct examples where this produces a space of dimension $\Theta(2^k)$).  Unfortunately, without any additional assumptions, even just the consistency problem is now NP-hard.  That is, 
given a collection of conjunctions, it is NP-hard to determine whether there exist $k$ monomials such that each can be written as a product of subsets of those monomials (it is called the ``set-basis problem'' \cite{GJ79}).  For this reason, we will make a natural {\em anchor-variable} assumption that each metafeature $m_i$ 
has at least one variable (call it $y_i$) that is not in
any other metafeature $m_j$.  So this is a generalization of the disjoint case
where {\em every} variable in $m_i$ is not
inside any other $m_j$.  We can think of $y_i$ as an ``anchor variable'' for metafeature $m_i$.

We now show how with this assumption we can efficiently solve the consistency problem (and  {\em find} the smallest set of monomials for which one can reconstruct each target).  Using this as a subroutine, we then show how to solve an abstract online learning problem where at each stage we must propose a set of at most $k$ monomial metafeatures and then pay a cost of 1 if the next target cannot be written as a product over them.  This can then be applied to give efficient life-long learning of related conjunctions over product distributions.
In Section \ref{sec:sparseboolean} we give an application to constructing Boolean superimposition-based autoencoders.  We then relax the anchor-variable assumption and show how under this relaxed condition we can solve for near-optimal {\em sparse} autoencoders as well as life-long learning of conjunctions under relaxed conditions. In 
Section \ref{se:poly}, we build on some of these results to give an algorithm for life-long learning of polynomials.

\subsection{Solving the Consistency Problem}

We now show that we can use Algorithm~\ref{algo-consistency-anchor} (below) for solving the consistency problem under the anchor variable assumption.  That is, given a collection of conjunctions, the goal is to find the fewest monomial metafeatures needed to reconstruct all of them as products of subsets of the metafeatures.  
Given a conjunction $T$ we denote by $\vars(T)$ the variables appearing in $T$.
Given a variable $z$ and a set of conjunctions $\TS$ we denote by $N(\TS,z)$ the set of conjunctions in $\TS$ that contain $z$.

\setcounter{algorithm}{2} 
\begin{algorithm}
{\bf Input}: set $\TS = \{T_1,  \ldots, T_{\labell} \}$ of conjunctions.
\begin{enumerate}
\vspace*{-0.1in}\addtolength{\itemsep}{-0.1in}
\item Let $i=0$.
\item Let $h(T)$ denote the conjunction of all metafeatures $\thh_j$ produced so far that are fully contained in $T$. I.e., $\vars(h(T)) =\cup \{ \vars(\thh_j): \vars(\thh_j) \subseteq \vars(T) \}$.
\item While there exists $T \in \TS$ s.t. $\vars(T) \neq \vars(h(T))$ do:
\begin{itemize}
\item[(1)]  Let $T$ be the target of least index in $\TS$ s.t. $\vars(T) \neq \vars(h(T))$.
\item[(2)] Choose  $z_{i+1}$ to be a minimal variable in $\vars(T) \setminus \vars(h(T))$; that is, there is no other variable $z' \in \vars(T) \setminus \vars(h(T))$ s.t. $N(\TS,z') \subset N(\TS,z)$.
    If there are multiple options, choose $z_{i+1}$ to be the option of least index.
\item[(3)] Let $\vars(\thh_{i+1})$ be the intersection of $\vars(T)$ for all $T$ in $\TS$ that contain $z_{i+1}$. 
That is $\vars(\thh_{i+1})=\bigcap\limits_{T \in \TS, z_{i+1} \in \vars(T)} \vars(T).$
 \item [(4)] i=i+1
  \end{itemize}
\end{enumerate}
{\bf Output}: Conjunctions $\thh_1, \ldots, \thh_i$ s.t. each $T_j$ is a conjunction of a subset of them.
\caption{Consistency problem for monomial metafeatures with anchor variables \label{algo-consistency-anchor}}
\end{algorithm}

\begin{lemma}\label{lem:conjunctionconsistency}
Let $\TS$ be a set of conjunctions such that each of them is a conjunction of some subset of metafeatures $m_1, \ldots, m_k$ satisfying the anchor variable condition.
We can use Algorithm~\ref{algo-consistency-anchor} to find $\thh_1, \ldots, \thh_i$, $i \leq k$  s.t. each $T_j \in \TS$ is a conjunction of a subset of them. Moreover each $\thh_i$ is associated to a metafeature $m_{t_i}$ s.t. the following conditions are satisfied:
\begin{enumerate}\vspace*{-0.1in}\setlength{\itemsep}{2pt}\setlength{\parsep}{0pt}\setlength{\parskip}{0pt}
\item [(a)] 
$\vars(m_{t_i})  \subseteq vars(\thh_i)$; that is,  $\thh_ i$ is more specific than $m_{t_i}$.
\item [(b)] For all targets $T$ in $\TS$ such that $\vars(m_{t_i}) \subseteq \vars(T)$ we have $\vars(\thh_i)\subseteq vars(T)$; that is,
$\thh_i$ is not too specific.
\item[(c)] For any $j$, if $y_j \in \vars(\thh_i)$  then $\vars(m_j) \subseteq \vars(\thh_i$).
\end{enumerate}
\end{lemma}

\begin{proof}
Note that for any $i$  for any $T \in \TS$ we have $\vars(h(T)) =\cup \{ \vars(\thh_j): j\leq i, \vars(\thh_j) \subseteq \vars(T) \}$; that is, $\vars(h(T))$ represents all variables  from $T$ that are already used by the previous
hypothesized metafeatures $\thh_j$ whose relevant variables are contained in $T$.

We prove the desired statement by induction.
Assume inductively that $\thh_1, \ldots, \thh_i$ satisfy conditions (a),(b),(c).
We show that  $\thh_{i+1}$ satisfies these conditions as well.

Consider the target $T$ we choose in step 3(1) in round $i+1$. We know $z_{i+1} \in \vars(T) \setminus \vars(h(T))$ and
that $T$ is a conjunction of the true metafeatures. So $z_{i+1}$ belongs to some metafeature $m_{t_{i+1}}$ s.t.
 $\vars(m_{t_{i+1}}) \subseteq \vars(T)$ .
 From the induction hypothesis, by conditions $(a)$,$(b)$ we know that  $m_{t_{i+1}} \neq m_{t_{i'}}$ for $i' \leq i$.
To see this assume by contradiction that $m_{t_{i+1}} = m_{t_{i'}}$ for $i' \leq i$; so $z_{i+1} \in m_{t_{i'}}$.
By condition $(a)$ we know  $\vars(m_{t_{i'}}) \subseteq \vars(\thh_{t_{i'}})$ and since $\vars(m_{t_{i'}}) \subseteq \vars(T)$ by condition $(b)$
we have $\vars(\thh_{t_{i'}}) \subseteq \vars(T)$, so $z_{i+1} \in \vars(h(T))$, contradiction.

Consider $T \in \TS$ such that  $\vars(m_{t_{i+1}}) \subseteq \vars(T)$. Since   $z_{i+1} \in \vars(m_{t_{i+1}})$ and
we create $\thh_{{i+1}}$ by intersecting the variables in every target $T$ containing $z_{i+1}$, we clearly have $\vars(\thh_{i+1}) \subseteq \vars(T)$, satisfying condition (b).
Also if any target $T$ contains an anchor variable $y_j$, then it must contain $m_j$, so condition (c) is satisfied as well.

We now show that (a) is satisfied, namely that $ \vars(m_{t_{i+1}})  \subseteq \vars(\thh_{i+1})$.
This could only fail if $z_{i+1}$ is not an anchor for $m_{t_{i+1}}$, so in step $2$ of the algorithm we intersected some target $T$ that contains $z_{i+1}$ but does not contain $m_{t_{i+1}}$. This can only happen if $z_{i+1}$ also belongs to some other $m_j$.
But then $z_{i+1}$ is not minimal since $y_{t_{i+1}}$ (the true anchor variable for $m_{t_{i+1}}$, which is also contained in $\vars(T) \setminus \vars(h(T))$ by (c))
 satisfies $N(\TS,y_{t_{i+1}}) \subset N(\TS,z_{i+1})$, and so would have been chosen instead of $z_{i+1}$ in step $3(1)$.
\end{proof}

\subsection{An Abstract Online Problem}
Building on Algorithm \ref{algo-consistency-anchor} and Lemma \ref{lem:conjunctionconsistency}, we now describe an algorithm for the following abstract online setting.  At each time-step $\labell$ we propose a set $\tM$ of at most $k$ hypothesized metafeatures and are provided with a target conjunction $T_\labell$.  If $T_\labell$ can be written as a conjunction of metafeatures in $\tM$ then we pay 0.  If not, then we pay 1 and may update our set $\tM$ using $T_\labell$ (this corresponds to the case of learning $T_\labell$ from scratch).  Our goal is to bound our total cost, under the assumption that there {\em exists} a set of $k$ metafeatures for all targets.  To do so we need to argue that each time we pay 1, 
we can use $T_\labell$ to make progress.

\begin{algorithm}
{\bf Input}: Targets $T_1, T_2, \ldots, T_m$ provided online.
\begin{enumerate}\vspace*{-0.1in}\setlength{\itemsep}{2pt}\setlength{\parsep}{0pt}\setlength{\parskip}{0pt}
\item Initialize $\TS = \emptyset$ and $\tM = \emptyset$.
\item For $\labell=1 $ to $m$ do:
\begin{itemize}
\item If we cannot represent $T_{\labell}$ as conjunction of hypothesized metafeatures $\tM$ then
\begin{enumerate}
   \item[$\bullet$] Add $T_{\labell}$ to $\TS$.
   \item[$\bullet$] Run Algorithm~\ref{algo-consistency-anchor} with input $\TS$ to produce hypothesized metafeatures $\tM$.
   \end{enumerate}
  \end{itemize}
\end{enumerate}
\vspace*{-0.1in}
{\bf Output}: Hypothesized metafeatures $\tM$.
\caption{Lifelong Learning of Conjunctions with Monomial Metafeatures  \label{algo-monomial-anchor-online}}
\end{algorithm}

\begin{theorem}
\label{thm-online-anchor}
 The number of targets that need to be learned from scratch in in Algorithm~\ref{algo-monomial-anchor-online} is at most $n^2 +k$.
\end{theorem}

\begin{proof}
For any given  set of targets $\TS$  learnt from scratch, we define a directed graph $G_{\TS}$  on the variables, by adding an edge $(x_i,x_j)$ if 
every target in $\TS$ that has $x_i$ also has $x_j$. Note that if
$\TS \subseteq \TTS $ we have  $E(G_{\TS}) \subseteq E(G_{\TTS})$.
We start with the complete directed graph (corresponding to $\TS=\emptyset$), and then we argue that each time we are forced to learn a new target from scratch and increase $\TS$ we either delete at least one edge from the graph or we increment the number of hypothesized metafeatures by $1$.

Suppose the new target $T_{\labell}$ cannot be represented using the current hypothesis metafeatures.  So we add $T_{\labell}$ into $TS$ and re-run Algorithm~\ref{algo-consistency-anchor} .  Let us look at the first time the new run differs from the old run. There are three possibilities for this difference.

(1) It could be that we choose a different $z_{i+1}$ in step $3(2)$ of  Algorithm~\ref{algo-consistency-anchor}.  There are two ways this can happen: (a) the old $z_{i+1}$ is not minimal any more or $(b)$ it could be some $z'$ (of lower index than the old $z_{i+1}$) was not minimal before but is minimal now.
 In case (a) we have some $z'$ is now in a strict subset of the targets in $\TS$ that contain $z_{i+1}$ but this was not the case before adding $T_{\labell}$.
This means the new target $T_{\labell}$ must contain the old $z_{i+1}$ but not $z'$, and all previous targets that contained either $z'$ or $z_{i+1}$ contained both of them. 
That means we cut the edge $(z_{i+1},z')$.
In case $(b)$, some $z'$ (of lower index than the old $z_{i+1}$) was not minimal before but is minimal now.
This means that before there was some $z''$ that was in a strict subset of the targets as $z'$, but it is not anymore.
Now, $z'$ is minimal, $z''$ is no longer in a strict subset of the targets containing $z'$; so the new target contains $z''$ but not $z'$. So we  cut the edge $(z'',z')$.

(2) It could be that we get the same $z_{i+1}$ but different $\thh_{i+1}$ in step $3(3)$; this means  $\vars(\thh_{i+1})$ is smaller.
Thus we cut the edges between $z_{i+1}$ and all the variables in the old $\thh_{i+1}$ that are not in the new $\thh_{i+1}$.

(3) It could be that we use the new target $T_{\labell}$ in step $3(1)$. Since we go through the targets in order,  the only way that the first difference can be when the new target is used in 3(1) is if every previous metafeature is created the same as before. Therefore, in this case we create a new metafeature. So, the number of metafeatures is increasing and we make progress as desired.
\end{proof}

\subsection{Applications}
As one immediate application of the above abstract online problem, since conjunctions over $\{0,1\}^n$ can be exactly learned in the Equivalence Query model with at most $n$ equivalence queries (and conjunctions over $\{0,1\}^k$ can be learned from at most $k$ equivalence queries), we immediately have the following:
\begin{corollary}\label{cor:EQ}
Let $\TS$ be a sequence of $m$ conjunctions such that each is a conjunction of some subset of metafeatures $m_1, \ldots, m_k$ satisfying the anchor variable condition.  Then this sequence can be learned using only $O(mk + n^3)$ equivalence queries total.
\end{corollary}
As another application of the above abstract online problem, we now show we can learn with good sample complexity over any product distribution $D$.

\begin{theorem}
\label{thm-product-anchor}
Assume that all $D_{\labell} = D$ which is a product distribution, that the metafeatures $m_i$ satisfy the anchor variable assumption and  all the target functions $c_{\labell}$ are balanced.
We can learn hypotheses $h_1, \ldots, h_m$ of error at most $\epsilon$ by using  Algorithm~\ref{algo-monomial-anchor} with parameters
$s_1(n,\epsilon,\delta)= O(n/\epsilon \log(n/\delta))$,
$s_2(n,\epsilon,\delta)=k/\epsilon \log(m/\delta)$, and
$s_3(n,\epsilon,\delta)=n/\epsilon \log(n k/\delta)$. 
The total   number of labeled examples needed is $\tO((n^2+k)n/\epsilon \log(n/\delta) + km/\epsilon)$.
\end{theorem}
\begin{proof}
Let us call a variable $i$ {\em insignificant} if over a sample of size
$\Theta((n/\epsilon)\log(n/\delta))$
appears set to $0$ less than $\epsilon/4n$ fraction of the time.
 Let $I$ be the set of insignificant variables and let $S$ be the set of significant variables.
Let $D_S$ be the distribution $D$ restricted to examples that are set to $1$ on all  variables in $I$.
We can show that  error at most $\epsilon/2$ over $D_S$ implies error at most $\epsilon$ over $D$.
This is true, since by Chernoff bounds for every variable  $i$ we have $\Pr_{x \sim D}{[x_i=0]}\leq \epsilon/2n$ if $i$
appears set to $0$ less than $\epsilon/4n$ fraction of the time over a sample of size $\Theta(n \log(n)/\delta)$ . So,
  by union bound $\Pr_{x \sim D}{[\exists i \in I, x_i=0]}\leq \epsilon/2$.

It remains to show that hypotheses $h_1, \ldots, h_m$ have error at most $\epsilon/2$ over $D_S$.
First note that for any label $\labell$ if $x_i \notin c_{\labell}$ and $i \in S$, then
     $\Pr_{x \sim D_S}[x_i=0|c_{\labell}(x)=1] = \Pr_{x \sim D_S}[x_i=0]$.
This follows from two facts.  First, since the target $c_{\labell}$ is a conjunction we have
     $\Pr_{x \sim D_S}[x_i=0|c_{\labell}(x)=1] = \Pr_{x \sim D_S}[x_i=0| x_j=1 \forall x_j \in c_{\labell}]$.
Second, because $D$ is a product distribution and $D_S$ be the distribution $D$ restricted to examples that are set to $1$ on all  variables in $I$, we have
     $\Pr_{x \sim D_S}[x_i=0| x_j=1 \forall x_j \in c_{\labell}] = \Pr_{x \sim D_S}[x_i=0]$.
  Furthermore since $c_{\labell}$ is balanced over $D$ and so over $D_S$ we get   $\Pr_{x \sim D_S}[x_i=0, c_{\labell}(x)=1] \geq c \epsilon/n$.

Note that every time we learn  we learn a problem from scratch (by using the original variables), we get
 $n/\epsilon \log(n/\delta)$ labeled examples from $D_S$. Therefore significant variables that are not in the target will appear set to $0$
 in at least one positive example.
Therefore for every problem $i$ learned based on the original features (via case 1 or 3(b)), we learn the target, that is $h_i=c_i$.

These together with the argument in the Theorem~\ref{thm-online-anchor} gives the desired result.
\end{proof}

\begin{algorithm}
{\bf Input}: parameters $n$,$m$,$k$, $\epsilon$, $\delta$; $s_1(n,\epsilon,\delta)$, $s_2(n,\epsilon,\delta)$, $s_3(n,\epsilon,\delta)$,
access to unlabeled examples from $D_i$ and label oracles for problems $\labell \in \{1, \ldots, m\}$, .

\begin{enumerate}
\item Draw  $s_1(n,\epsilon,\delta)$ 
unlabeled examples and identify the set of  variables $I$ that are set to $0$ less than $\epsilon/4n$ fraction of the times.
\item Draw a set $S_1$ of $s_1(n,\epsilon,\delta)$ 
examples from
$D_1$, remove from $S_1$ those examples for which not all features in
$I$ are set to 1.
Label $S_1$ according to problem $1$.
 Find  a conjunction $h_1$ consistent with $S_1$. Initialize $\TS=\{h_1\}$.
 \item Run Algorithm~\ref{algo-consistency-anchor} with input $\TS$ to produce hypothesized metafeatures $\tM$.
\item For the learning problem  $\labell=2 $ to $m$

\begin{itemize}
\item  Draw a set $S_{\labell}$ of $s_2(n,\epsilon,\delta)$,
 examples from
$D_{\labell}$, remove from $S_{\labell}$ those examples in $S_{\labell}$ for which not all
features in $I$ is set to $1$;
re-represent each example in $S_{\labell}$ using meta-features in $\tM$ and check if we can find a conjunction consistent with $S_{\labell}$,
\begin{itemize}
\item[(a)] If yes,   let $h_{\labell}$ be its representation over the original features and record it.
\item[(b)] If not,   draw a set $S_{\labell}$ of $s_3(n,\epsilon,\delta)$,
 examples from $D_{\labell}$, remove from $S_{\labell}$ those
 examples for which a feature in $I$ is set to $1$; find  a conjunction $mh_{\labell}$ consistent with $S_{\labell}$.
   \begin{enumerate}
   \item[$\bullet$] Add $h_{\labell}$ to $\TS$.
   \item[$\bullet$] Run Algorithm~\ref{algo-consistency-anchor} with input $\TS$ to produce hypothesized metafeatures $\tM$.
   \end{enumerate}

  \end{itemize}
\end{itemize}
\end{enumerate}
{\bf Output}: Conjunctions $h_1, \ldots, h_m$.
\caption{Transfer Learning of Conjunctions with Monomial Metafeatures}
\label{algo-monomial-anchor}
\end{algorithm}

\subsection{Sparse Boolean Autoencoders and Relaxing the Anchor-Variable Assumption}

The above results (and in particular, Lemma \ref{lem:conjunctionconsistency}) have an interesting interpretation as constructing a minimal feature space for Boolean, or superimposition-based, autoencoding. 

Specifically, consider a collection of black-and-while pixel images $\{T_\labell\}$ where each $T_\labell \in \{0,1\}^n$.   Our goal is to contruct a 2-level auto-encoder ${\cal A}$ (for each $\labell$, we want ${\cal A}(T_\labell) = T_\labell$) with as few nodes in the middle (hidden) level as possible, such that nodes in the hidden level compute the AND of their inputs, and nodes in the output level compute the OR of their inputs.  We can view each hidden node in such a network as representing a ``piece'' of an image, with the autoencoding property requiring that each $T_\labell$ should be equal to the bitwise-OR of all pieces contained within it (i.e., superimposing them together).  Formally, for each hidden node $j$,  let $m_j \in \{0,1\}^n$ denote the indicator vector for the set of inputs to that node (which without loss of generality will also be the set of outputs of that node), and say that $m_j \preceq T_\labell$ if each bit set to 1 in $m_j$ is also set to 1 in $T_\labell$;  we then require $T_\labell$ to be the bitwise-OR of all $m_j \preceq T_\labell$.  Lemma \ref{lem:conjunctionconsistency} then shows that given a collection of images $\{T_\labell\}$, Algorithm \ref{algo-consistency-anchor} finds the smallest number of hidden nodes needed to perform this autoencoding, under the assumption that each metafeature $m_j$ contains some anchor-variable (some pixel set to 1 that no other metafeature sets to 1).

\label{sec:sparseboolean}
We now consider the problem of {\em sparse} Boolean autoencoding.  That is, given a set $\TS = \{T_\labell\}$, with each $T_\labell \in \{0,1\}^n$, our goal is to find a collection of metafeatures $\thh_j$ (perhaps more than $n$ of them) such that each  $T_\labell \in \TS$ can be written as the bitwise-OR of at most $k$ of the $\thh_j$ (where $k \ll n$).  Clearly this is trivial by having one metafeature $\thh_j$ for each $T_\labell$, so our goal will be to have the (approximately) fewest of them subject to this condition.  Additionally, because we want sparse reconstruction, we want for each $T_\labell$ that $|\{j: \thh_j \preceq T_\labell\}|$ should be small as well.  

This problem has two motivations.  From the perspective of autoencoding, this corresponds to finding a sparse autoencoder (viewing the $T_\labell$ as pixel images). From the perspective of life-long learning, if this can be done online then (viewing the $T_\labell$ as conjunctions)  it will allow for fast learning, since conjunctions of $k$ out of $N$ variables can be learned with sample complexity (or equivalence queries) only $O(k \log N)$; in this case we would actually not need the additional ``sparse reconstruction'' property above.

To solve this problem, we make a relaxed version of the anchor-variable assumption (anchor-variables no longer make sense once the number of metafeatures exceeds the number of input features $n$) which is that each metafeature should have a {\em set} of $\leq c$ variables (for some constant $c$) such that any $T_\labell$ containing that set should have the metafeature as one of its $k$ ``relevant metafeatures''.  We call this the {\em $c$-anchor-set} assumption. Note that metafeatures satisfying the anchor-{\em variable} assumption will also satisfy the $c$-anchor-set assumption for $c=1$.  Note also that in general the $c$-anchor-set assumption is a requirement on both the metafeatures and on the set $\TS$.   Formally, we make the following definition:
\begin{definition}\label{def:anchorset}
A set of metafeatures $M=\{m_j\}$ and set of targets $\TS=\{T_\labell\}$ satisfy the {\em $c$-anchor-set} assumption at sparsity level $k$ if 
\begin{enumerate}\vspace*{-0.1in}\setlength{\itemsep}{1pt}
\item for each $T_\labell \in \TS$ there exists a set $R_\labell$ of at most $k$ ``relevant'' metafeatures in $M$ such that $T_\labell$ is the bitwise-OR of the metafeatures in $R_\labell$, and
\item For each $m_j \in M$ there exists $y_j \preceq m_j$ of Hamming weight at most $c$ such that for all $\labell$, if $y_j \preceq T_\labell$ then $m_j \in R_\labell$.  Note that in particular this implies that $|\{j:m_j \preceq T_\labell\}| \leq k$.
\end{enumerate}
\end{definition}
We now prove that under this assumption, we can solve for a near-optimal set of metafeatures $\{\thh_j\}$.
\begin{theorem}\label{thm:sparseauto}
Given a set of targets $\TS = \{T_\labell\}$ in $\{0,1\}^n$, suppose there exists a set of metafeatures $M$ satisfying the $c$-anchor-set assumption at sparsity level $k$.  Then in time $poly(n^c)$ we can: 
\begin{enumerate}\vspace*{-0.1in}\setlength{\itemsep}{1pt}
\item Find a set of $O(n^c)$ metafeatures such that each $T_\labell \in \TS$ can be written as the bitwise-OR of at most $k$ of them, and
\item Find a set of $O(|M|\log(n|\TS|))$ metafeatures that satisfy the $c$-anchor-set assumption with respect to $\TS$  at sparsity level $O(k\log(n|\TS|))$.  
\end{enumerate}
\end{theorem}
\begin{proof}
Item (1) is the easier of the two.  For each $y \in \{0,1\}^n$ of Hamming weight at most $c$, define $\thh_y$ to be the bitwise-AND of all $T_\labell \in \TS$ such that $y \preceq T_\labell$.  By definition of the anchor-set assumption, for each $m_j \in M$ there exists $y_j \preceq m_j$ of Hamming weight at most $c$ such that for all $\labell$, if $y_j \preceq T_\labell$ then $m_j \in R_\labell$. Therefore we have both (a) $m_j \preceq \thh_{y_j}$ and (b) $\thh_{y_j} \preceq T_\labell$ for all $\labell$ such that $m_j \in R_\labell$.  Therefore each $T_\labell$ is the bitwise-OR of the (at most $k$)  metafeatures $\thh_{y_j}$ such that $m_j \in R_\labell$.

For item (2), we begin by creating $O(n^c)$  metafeatures $\thh_y$ as above.  We next set up a linear program to find an optimal fractional subset of these metafeatures, and then round this fractional solution to a set of metafeatures  $\tilde{M}$ satisfying (2).  Specifically, the LP has one variable $Z_y$ for each $\thh_y$ with objective
\begin{eqnarray*}
&\mbox{Minimize:} & \sum_y Z_y, \\
\mbox{Subject to : } 
(1) & \mbox{for all $y$:}& 0 \leq Z_y \leq 1\\
(2) & \mbox{for all $\labell, i$:}&\textstyle  \sum_{y: e_i \preceq \thh_y \preceq T_\labell} Z_y \geq 1 \mbox{\ \ \ ($e_i$ is the unit vector in coordinate $i$)}\\
(3) & \mbox{for all $\labell$:} & \textstyle \sum_{y: \thh_y \preceq T_\labell} Z_y \leq k
\end{eqnarray*}
Here, constraint (2) ensures that each $T_\labell$ is fractionally covered by all the metafeatures contained inside it, and constraint (3) ensures that each $T_\labell$ fractionally contains at most $k$ metafeatures.  Note also that setting $Z_{y_j}=1$ for each $m_j \in M$ (and setting all other $Z_y=0$) satisfies all constraints at objective value $|M|$.  

We now produce our output set of metafeatures $\tilde{M}$ by independently rounding each $Z_y$ to 1 with probability $\min[1, Z_y\ln(n^2|\TS|)]$.  Clearly $\E[|\tilde{M}|] = O(|M|\log(n|\TS|))$ so the key issue is the coverage of each $T_\labell$ and the size of the set $\tilde{R}_\labell = \{\thh_y \in \tilde{M}: \thh_y \preceq T_\labell\}$.  Note that item (2) of Definition \ref{def:anchorset} will be satisfied by how the $\thh_y$ were constructed (taking the bitwise-AND of all $T_\labell$ such that $y \preceq T_\labell$).  First, for coverage, for each $\labell$ and $i$ such that variable $i$ is set to 1 by $T_\labell$, the probability that $\tilde{M}$ does not contain some $\thh_y$ such that $e_i \preceq \thh_y \preceq T_\labell$ is maximized when constraint (2) is satisfied at equality and all associated $Z_y$ are equal (by concavity).  This in turn is at most $\lim_{\epsilon \rightarrow 0} (1-\epsilon\ln(n^2|\TS|))^{1/\eps} = 1/(n^2|\TS|)$.  Thus, by the union bound, the probability that any $T_\labell$ fails to be completely covered by $\tilde{R}_\labell$ is at most $1/n$.   Now, to address the size of the sets $\tilde{R}_\labell$, the expected size of each $\tilde{R}_\labell$ by constraint (3) and the rounding step is at most $k\ln(n^2|\TS|) \leq \max[k,3]\ln(n^2|\TS|)$. By Chernoff bounds, the probability any given $\tilde{R}_\labell$ has size more than twice this value is at most $e^{-\max[k,3]\ln(n^2|\TS|)/3} \leq 1/(n^2|\TS|)$.  So, by the union bound, the probability that any $\tilde{R}_\labell$ is too large is at most $1/n^2$.
\end{proof}

Theorem \ref{thm:sparseauto} shows that we can efficiently find a near-optimal sparse autoencoder for any set of targets in $\{0,1\}^n$ having an optimal encoder satisfying the $c$-anchor-set assumption for constant $c$.
Theorem \ref{thm:sparseauto} also has the following corollary for online learning from equivalence queries, similar to Corollary \ref{cor:EQ}.

\begin{corollary}\label{cor:sparseauto}
Let $\TS$ be a sequence of $m$ conjunctions for which there exists a set $M$ of conjunctive metafeatures  satisfying the $c$-anchor-set assumption at sparsity-level $k$ for some constant $c$.  Then this sequence can be efficiently learned using only $O(mk\log(n) + n^2|M|)$ equivalence queries total.
\end{corollary}

\begin{proof}
We instantiate $O(n^c)$ metafeatures $\thh_y$, one for each $y \in \{0,1\}^n$ of Hamming weight at most $c$, setting each $\thh_y$ initially to the conjunction of all variables.  Given a new target $T_\labell$, we try to learn it as a conjunction of at most $k$ of these metafeatures using at most $O(k\log n^c)$ equivalence queries using the Winnow algorithm.  If we are unsuccessful, we learn $T_\labell$ from scratch using at most $n$ equivalence queries. We then (viewing $T_\labell$ and the $\thh_y$ as their indicator vectors) let $\thh_y \leftarrow \thh_y \; \& \; T_\labell$ (where ``$\&$'' denotes bitwise-AND) for all $\thh_y$ such that $y \preceq T_\labell$.  This maintains the invariant that for each $m_j \in M$, we have $m_j \preceq \thh_{y_j}$, which implies that each time we learn some $T_\labell$ from scratch we shrink at least one $\thh_{y_j}$ by at least one variable.  This can happen at most $n|M|$ times.
\end{proof}

\section{Life-long Learning of Polynomials}
\label{se:poly}
We now show an application of the results in Section~\ref{se:monomial} to the case where the target functions are polynomials from $\{0,1\}^n$ to $\reals$, whose terms ``share'' a not too large number of pieces. Specifically, we assume there exist $k$ distinguished monomials (which might overlap) such that each monomial in each target polynomial can be written as a product of some subset of them.  For example, if our distinguished monomials are $\{x_1x_2x_3, x_3x_4x_5, x_5x_6x_7, x_7x_8x_1\}$ then we might have polynomials such as $4x_3x_4x_5x_6x_7 - 2x_5x_6x_7x_8x_1$ and $3x_1x_2x_3x_4x_5 + 3x_1x_2x_3x_7x_8$.  If the target polynomials use $r$ distinct monomials in total, then viewed as a network we have $k$ nodes in a first hidden layer, where each is a {\em product} of some of the inputs, $r$ nodes in a second hidden layer, where each is a {\em product} of outputs of the first hidden layer, and then the final outputs (our target functions) are weighted {\em linear} functions of the second hidden layer.
Efficiently learning polynomials requires membership queries (under the assumption that juntas are hard to learn) in addition to equivalence queries or random examples even in the single task setting~\cite{ss93}.
 So we will assume access to membership queries as well.  However, our goal will be to use these sparingly, only when we need to learn a new function from scratch.   When learning from scratch we use an algorithm of Schapire and Sellie \cite{ss93} that learns polynomials exactly.  Any  function from $\{0,1\}^n$ to $\reals$  has a unique representation as a polynomial  over $\{0,1\}^n$, so learning exactly means learning the exact functional form of the target function as a polynomial.

As a warmup, let's  first consider a simple case.  Assume that the target functions are polynomials that simply use at most $k$ distinct monomials in total.  
This corresponds to a network with with only one hidden layer of $k$ nodes.
In this case, there is a very simple algorithm that exploits the structure of the problem. Let $\tM$ be the set of hypothesized monomials. Given a new target function,  we try to learn a linear function over the monomials in  $\tM$.
If we succeed, we are done and move on to the next problem. If not, we learn from scratch using queries; we will clearly get at least one new monomial we have not seen, and add it to the set $\tM$.
So, we only need to learn $k$ problems from scratch.

We now  provide an algorithm for the general, more interesting case.
Our theoretical guarantees are under the assumptions that each polynomial in our family has  $L_1$ norm bounded by $B$ and the number of terms in each is bounded by $t$. If the target function has an $L_1$ norm  bounded by $B$ and its monomials can indeed be written as products of our metafeatures, then by considering all products of metafeatures and running an $L_1$-based algorithm for learning linear functions \cite{LLW95}, we can achieve low mean squared error using only $O(B^2 \log(2^k)) = O(B^2k)$ examples.

\begin{algorithm}
{\bf Input}: $n$,$m$.
\begin{enumerate}
\item Let $\tM=\emptyset$.  $\tM$ is the set of hypothesized metafeatures for the first hidden layer.\\
 Let $TS=\emptyset$.  $TS$ is the set of terms used to create the hypothesized metafeatures in  $\tM$.
\item For the learning problem  $\labell=1 $ to $m$
\begin{itemize}
\item[(a)] Create the set $\tTS$ of terms obtained by taking all possible conjunctions of the hypothesized metafeatures in $\tM$.
\item[(b)] Attempt to learn problem $\labell$ as a linear function over the terms in $\tTS$ to low mean squared error (quadratic loss) using $O(B^2k)$ examples.
\begin{enumerate}
   \item[$\bullet$] If we succeed, record the hypothesis.
   \item[$\bullet$] Otherwise, run the algorithm of Schapire and Sellie \cite{ss93} to learn  the target  $T_{\labell}$ for problem $\labell$ exactly  based on the original feature representation with equivalence and membership queries.
       \begin{enumerate}
   \item  Expand $TS$ by adding  any term in $T_{\labell}$ that was not in $TS$.
   \item Run  Algorithm~\ref{algo-consistency-anchor} with input $TS$  to ``compactify'' it into the fewest number of (possibly overlapping) conjunctive metafeatures that can be used to recreate all the terms in $TS$. Let  $\tM$ be the resulting metafeatures.
   \end{enumerate}
   \end{enumerate}
  \end{itemize}
\end{enumerate}
{\bf Output}: Hypothesis functions of low error for each learning task.
\caption{Multi-task  learning for polynomial target functions}
\label{algo-polynomials}
\end{algorithm}
\begin{theorem}  Assume that the monomials corresponding to the first network  layer satisfy the anchor assumption and the $L_1$ norm of the target polynomials is bounded by $B$. Consider running  
Algorithm~\ref{algo-polynomials}.
The number of targets needed to learn from scratch  is $n^2+k$. Furthermore the number of  hypothesized metafeatures satisfies $|\tM| \leq k$ at any time, thus the sample complexity of learning problems in Step 2(b) is $O(B^2k)$ per problem.
\end{theorem}

\begin{proof}
In Algorithm~\ref{algo-polynomials}, $\tM$  represents the set of hypothesized metafeatures for the first hidden layer -- they are learned using Algorithm~\ref{algo-consistency-anchor}; let $k'=|\tM|$.
Let $\tTS$=all possible conjunctions of hypothesized metafeatures in $\tM$; so $TS \subseteq \tTS$, $|\tTS|= 2^{k'}$.

We know that in the true underlying network  the metafeatures in the first middle layer are monomials satisfying the anchor assumption and the metafeatures in the second middle layer are monomials of meta-features in the first layer. Note that every time we fail to learn in Step 2(b) we know that at least one of the monomials that can make up the target polynomial (which is a metafeature second level of the true network) cannot be written as a conjunction of hypothesized first level metafeatures  $\tM$. Since we create $\tM$ by using  Algorithm~\ref{algo-consistency-anchor}, by Theorem~\ref{thm-online-anchor} we only need  to learn  at most $n^2 +k$ problems from from scratch (that is $|TS| \leq n^2 +k$), and furthermore, $k' \leq k$.
\end{proof}

Note that while the sample complexity of  Algorithm \ref{algo-polynomials} is linear in $k$ for problems learned from scratch, its {\em running time} is exponential in $k$, due to the work in creating the set  $\tTS$.  However, a poly$(k)$ bound seems unachievable because it would require solving the junta learning problem.  In particular, the problem of learning polynomials over $k$ metafeatures is at least as hard as  learning polynomials over $\{0,1\}^k$ (because even if the true metafeatures were given to us in advance,  one possibility is that the targets could be arbitrary polynomials over $x_1,\ldots,x_k$).    Thus, for this problem one should think of $k$ as small.  

\section{Discussion and Open Problems}
In this work we present algorithms for learning new internal representations when presented with a series of learning problems arriving online that share different types of commonalities.  For the case of linear threshold functions sharing  linear subspaces, we require log-concave distributions to ensure that error can be both upper-bounded and lower-bounded by some ``nice'' function of angle: the lower bound helps to ensure that the span of accurate hypotheses is close to the span of their corresponding true targets (though one must be careful with error accumulation), and the upper-bound ensures that a sufficiently-close approximation to the span of the true targets is nearly as good as the span itself.   It is an interesting question whether one can extend these results to distributions that do not have such properties while still maintaining the streaming nature of the algorithms (i.e., remembering only the learned rules and not the data from which they were generated).  For the case of product metafeatures, our results have natural interpretations as autoencoders, which interestingly do not require assumptions such as the problem matrix being incoherent or a generative model, only the anchor-variable or anchor-set assumption.  It would be interesting to see whether an analog of the anchor-set assumption could be applied to dictionary learning problems such as in \cite{AroraGM14}.

\paragraph{Acknowledgements}
\noindent  This work was supported in part by  NSF grants CCF-0953192, CCF-1451177,
 CCF-1422910,  IIS-1065251, ONR grant N00014-09-1-0751,  AFOSR grant
FA9550-09-1-0538,  and a Microsoft
Research Faculty Fellowship.

{
\bibliographystyle{plain}
\bibliography{transfer}
}

\appendix
\section{Proofs for halfspaces with more complex common structure}

\setcounter{algorithm}{1} 
\begin{algorithm}

{\bf Input}: $n$,$m$,$k$, access to labeled examples for problems $i \in \{1, \ldots, m\}$, parameters $\epsilon$,
 $\wepsacc$,  $\uepsacc$.
\begin{enumerate}
\addtolength{\itemsep}{-0.1in}
\item Learn the first target to error $\wepsacc$ to get an $n$-dimensional vector $\talpha_1$.
\item Set $\tw_1 = \talpha_1$, $\tk=1$, $\tu_1=(1)$, $\tr=1$,  $\tc_1=(1)$,
 and $i_1=1$.
\item For the learning problem  $i=2 $ to $m$
\begin{itemize}
\item Try to  learn $\smdim$-sparsely by using the $\tr$ dimensional representation given by the second level meta-features  $v \rightarrow \tU \tW v$.
I.e., check whether for  the learning problem $i$ there exists a $\smdim$-sparse hypothesis
$ \sgn(\alpha_{i,1} (\tu_1 \tW v)+ \cdots + \alpha_{i,\tr} (\tu_{\tr} \tW v))$ of error at most $\epsilon$.
\begin{itemize}
\item[(a)] If yes, set $\tc_i=(\alpha_{i,1}, \ldots, \alpha_{i,\tk} )$.
\item[(b)] Otherwise,  check whether for learning problem $i$ there exists a hypothesis in the $\tk$
dimensional representation given by the first level meta-features  $v \rightarrow \tW v$ of error at most $\uepsacc$,
i.e., a hypothesis
$ \sgn(\alpha_{i,1} (\tw_1 \cdot v)+ \cdots + \alpha_{i,\tk} (\tw_{\tk} \cdot v))$ of error $\leq \uepsacc$.
\begin{itemize}
\item[(a)] If yes, set $\tr=\tr+1$, $\tu_{\tr}=(\alpha_{i,1}, \ldots, \alpha_{i,\tk})$, $j_{\tr}=i$.
 Extend all rows of $\tC$ with one zero, set $\tc_i = e_{\tr}$.
\item[(b)] If not,  learn  a classifier $\talpha_i$  for problem $i$ of accuracy $\wepsacc$ by using the original features.
Set $\tk=\tk+1 $,  $i_{\tk}=i$ ,
  $\tw_{\tk} = \talpha_i$. Extend all rows of $\tU$ by one zero, set $\tr=\tr+1$,   $\tu_{\tr} = e_{\tk}$.   Extend
  all rows of $\tC$ with one zero, set $\tc_i = e_{\tr}$, $j_{\tr}=i$.
  \end{itemize}
  \end{itemize}
\end{itemize}
\item Let $\tW$ be an $\tk \times n$ matrix whose rows are $\tw_1, \ldots, \tw_{\tk}$;
let $\tU$ be an $\tr \times \tk$ matrix whose rows are $\tu_1, \ldots, \tu_{\tr}$;
and let $\tC$ be the matrix $m \times \tk$ matrix whose rows are
 $\tc_1, \ldots, \tc_{\tk}$. Compute $\tA=\tC \tU \tW$.
\end{enumerate}
{\bf Output}: $m$ predictors; predictor $i$ is $v \rightarrow \sgn(\tA_i \cdot v)$
\caption{Life-long Learning with two levels of linear shared metafeatures\label{algo-two}}
\end{algorithm}

We now provide the algorithm and proof for Theorem \ref{thm:linmetalogconc2levels}.
\setcounter{theorem}{1}
\begin{theorem}  Assume  all marginals $D_i$ are isotropic log-concave and the target functions satisfy the above conditions. Consider $\tgamma \leq c\epsilon$, 
 $\uepsacc \leq c\frac{\tgamma \epsilon}{\smdim}$, $\gamma \leq c\uepsacc$,
and $\wepsacc \leq c\frac{\gamma\uepsacc}{k}$
for (sufficiently small) constant $c>0$.
Consider running 
Algorithm~\ref{algo-two}
with parameters $\epsilon$, $\wepsacc$, and $\uepsacc$.
Then $\tk \leq k$ and $\tr \leq \smdim r$.
Moreover the total   number of examples needed to learn all the problems to error $\epsilon$
is $\tO(nk/{\wepsacc} + kr/{\uepsacc} + m \log(r)/\epsilon)$.
\end{theorem}

\begin{proof}
We divide problems in two types: problems of  type (a) are those for which we can learn a classifier of desired error at most
 $\epsilon$ by using the previously learnt metafeatures at the second middle level; the rest are of type (b).

For problems of type (a) we achieve  error at most $\epsilon$ by design.
For each problem $i$ of type (b) we have either opened a new row in $\tU$, and we have set $\tw_{\ttr}=\alpha_i$, where $\ttr$ is such that $j_{\ttr}=i$ or we have opened both a new row in $\ttr$ in $\tU$ and a new row $\ttk$ in $\tW$, and set $j_{\ttr}=i$ and $i_{\ttk}=i$. In both cases, by design and Lemma \ref{l:angle} (and the fact that $\wepsacc \leq \uepsacc$) we have $\theta(\tu_{\ttr} \tW,  a_{j_{\ttr}}) = O(\uepsacc)$; furthermore since $\tc_i = e_{\ttr}$ we also have $\theta(\ta_{i}, a_{i}) = O(\uepsacc)$.
Furthermore, for each $\tu_{\ttr}$ we create for a problem $j_{\ttr}$ we have that
  $\tu_{\ttr} \tW$ is $\tgamma$-far from the span of those vectors in $\{ \tu_1,...,\tu_{j_{\ttr-1}} \}$ whose corresponding targets lie in space $U_s$, where
  $U_s$ is one of the  $r$ $\smdim$-dimensional subspaces that $a_{j_{\ttr}}$ belongs to.
(Otherwise if $\tu_{\ttr}$  is $\tgamma$-close we would have been able to learn sparsely to error $\epsilon$ based on the second level metafeatures.)

Using this together with the fact that $\uepsacc = O(\frac{\tgamma\epsilon}{\smdim})$, we obtain (by Lemma~\ref{lem:subspace}) that once we have $\smdim$ second level  meta-features $\tu_{j_{l_1}}, \ldots, \tu_{j_{l_{\smdim}}}$ whose corresponding targets $a_{l_1}, \ldots, a_{l_{\smdim}}$ lie in the same $\smdim$-dimensional space $U_s$, we have
$$\theta(U_s, \span(\tu_{j_{l_1}} \tW, \ldots,\tu_{j_{l_{\smdim}}} \tW)) = O(\smdim \uepsacc/\tgamma) \leq  \epsilon.$$  
Therefore we will be able to learn based on second level metafeatures any future target belonging to that subspace. This implies $\tr \leq \smdim r$.

Using the fact that $\wepsacc \leq c\frac{\gamma\uepsacc}{k}$,
as in the proof of Theorem~\ref{thm:linmetalogconc}, we can prove by induction that for each $\tw_{\ttk}$ we create for a problem $i_{\ttk}$, we have $a_{i_{\ttk}}$ is $\gamma$-far from
$\span\{a_{i_1}, \cdots ,a_{i_{\ttk-1}}\}$   and  $\tw_{\ttk}$ is $\gamma$-far from $\span(\tw_1,...,\tw_{\ttk-1})$; this implies $\tk \leq k$.
\end{proof}
\end{document}